\documentclass[letterpaper, 10 pt, conference]{ieeeconf}  % Comment this line out if you need a4paper

\IEEEoverridecommandlockouts                              % This command is only needed if 
\usepackage{times}  % DO NOT CHANGE THIS
\usepackage{helvet}  % DO NOT CHANGE THIS
\usepackage{courier}  % DO NOT CHANGE THIS
\usepackage[hyphens]{url}  % DO NOT CHANGE THIS
\usepackage{graphicx} % DO NOT CHANGE THIS
\usepackage{caption} % DO NOT CHANGE THIS AND DO NOT ADD ANY OPTIONS TO IT
\usepackage{algorithm}
\usepackage[noend]{algorithmic}
\usepackage{float}
\usepackage{newfloat}
\usepackage{listings}

\usepackage[utf8]{inputenc}
\usepackage{booktabs}
\usepackage{multirow}
\usepackage{subcaption} %-- This makes table captions normal looking and smaller
\usepackage{array}
\usepackage{xspace}
\usepackage{tabularx}
\usepackage{graphicx}
\usepackage{placeins}
\usepackage{centernot}
\usepackage{siunitx}
\usepackage{amsmath} % assumes amsmath package installed
\usepackage{mathtools}
\usepackage{amsfonts}
\usepackage{amssymb}  % assumes amsmath package installed

\usepackage{amsthm}  % assumes amsmath package installed
\usepackage{array}
\usepackage{eqparbox}
\usepackage{etoolbox}  % patch def of algorithmic environment
\usepackage{cancel}

\newcommand{\norm}[1]{\left\lVert#1\right\rVert}

\newcommand{\wastar}{wA*\xspace}
\newcommand{\pase}{PA*SE\xspace}

\newcommand{\epase}{ePA*SE\xspace}
\newcommand{\eas}{eA*\xspace}
\newcommand{\wepase}{\mbox{w-}\epase}
\newcommand{\insat}{INSAT\xspace}
\newcommand{\pinsat}{PINSAT\xspace}

\newcommand{\rrtconnect}{pBiRRT\xspace}
\newcommand{\trajopt}{TrajOpt\xspace}

\newcommand{\weas}{w-\eas\xspace}
\newcommand{\state}{\ensuremath{\mathbf{x}}\xspace}
\newcommand{\nextstate}{\ensuremath{\mathbf{x}'}\xspace}
\newcommand{\nextnextstate}{\ensuremath{\mathbf{x}''}\xspace}

\newcommand{\startstate}{\ensuremath{\state^S}\xspace}
\newcommand{\ac}{\ensuremath{\mathbf{a}}\xspace}
\newcommand{\dac}{\ensuremath{\mathbf{a^d}}\xspace}
\newcommand{\pac}{\ensuremath{\mathbf{a^d}}\xspace}
\newcommand{\Aset}{\ensuremath{\mathcal{A}}\xspace}
\newcommand{\goalreg}{\ensuremath{\mathcal{G}}\xspace}

\newcommand{\ed}{\ensuremath{e}\xspace}

\newcommand{\edge}{\ensuremath{(\state,\mathbf{a})}\xspace}

\newcommand{\pedgenext}{\ensuremath{(\state',\mathbf{a^d})}\xspace}
\newcommand{\open}{\ensuremath{\textit{OPEN}}\xspace}
\newcommand{\closed}{\ensuremath{\textit{CLOSED}}\xspace}
\newcommand{\be}{\ensuremath{\textit{BE}}\xspace}
\newcommand{\gval}{\ensuremath{g}}

\newcommand{\cost}{\ensuremath{c}\xspace}

\newcommand{\fval}{\ensuremath{f}}
\newcommand{\hval}{\ensuremath{h}}

\newcommand{\numexpanded}{\ensuremath{n\_successors\_generated}\xspace}
\newcommand{\graph}{\ensuremath{G}\xspace}
\newcommand{\vertex}{\ensuremath{v}\xspace}
\newcommand{\Vertices}{\ensuremath{\mathcal{V}}\xspace}
\newcommand{\Edges}{\ensuremath{\mathcal{E}}\xspace}

\newcommand{\numthreads}{\ensuremath{N_t}\xspace}

\newcommand{\lowdspace}{\mathcal{X}_L}
\newcommand{\auxspace}{\mathcal{X}_A}
\newcommand{\statespace}{\mathcal{X}}
\newcommand{\obsspace}{\mathcal{X}^{obs}}
\newcommand{\freespace}{\mathcal{X}^{\text{free}}}
\newcommand{\timemap}{\alpha}
\newcommand{\gpath}{\boldsymbol{\theta}}
\newcommand{\traj}{\boldsymbol{\phi}}
\newcommand{\trajdot}{\boldsymbol{\phi}^{(1)}}
\newcommand{\trajddot}{\boldsymbol{\phi}^{(2)}}
\newcommand{\trajdddot}{\boldsymbol{\phi}^{(3)}}
\newcommand{\utraj}{\boldsymbol{\psi}}

\newcommand{\kmin}{k_{\text{min}}}

\newcommand{\optimize}{\textsc{Optimize}\xspace}
\newcommand{\warmoptimize}{\textsc{WarmOptimize}\xspace}

\theoremstyle{definition}
\newtheorem{definition}{Definition}

\theoremstyle{definition}
\newtheorem{remark}{Remark}
\newtheorem{assump}{Assumption}
\newtheorem{lemma}{Lemma}
\newtheorem{theorem}{Theorem}
\newcommand\algorithmicprocedure{\textbf{procedure}}
\newcommand{\algorithmicendprocedure}{\algorithmicend\ \algorithmicprocedure}
\makeatletter
\newcommand\PROCEDURE[3][default]{%
  \ALC@it
  \algorithmicprocedure\ \textsc{#2}(#3)%
  \ALC@com{#1}%
  \begin{ALC@prc}%
}
\newcommand\ENDPROCEDURE{%
  \end{ALC@prc}%
  \ifthenelse{\boolean{ALC@noend}}{}{%
    \ALC@it\algorithmicendprocedure
  }%
}
\newenvironment{ALC@prc}{\begin{ALC@g}}{\end{ALC@g}}
\makeatother

\makeatletter
\patchcmd{\algorithmic}{\addtolength{\ALC@tlm}{\leftmargin} }{\addtolength{\ALC@tlm}{\leftmargin}}{}{}
\makeatother

\renewcommand{\baselinestretch}{0.97}

\title{\LARGE \bf
PINSAT: Parallelized Interleaving of Graph Search and \\ Trajectory Optimization for Kinodynamic Motion Planning}

\author{Ramkumar Natarajan$^{1}$, Shohin Mukherjee$^{1}$, Howie Choset$^{1}$ and Maxim Likhachev$^{1}$% <-this % stops a space
\thanks{$^{1}$ The authors are with The Robotics Institute at Carnegie Mellon University, Pittsburgh PA 15213. email: \{\tt\small rnataraj, shohinm, maxim, choset\}@cs.cmu.edu }}

\begin{document}

\maketitle
\thispagestyle{empty}
\pagestyle{empty}

\begin{abstract}
Trajectory optimization is a widely used technique in robot motion planning for letting the dynamics of the system shape and synthesize complex behaviors. Several previous works have shown its benefits in high-dimensional continuous state spaces and under differential constraints. However, long time horizons and planning around obstacles in non-convex spaces pose challenges in guaranteeing convergence or finding optimal solutions. As a result, discrete graph search planners and sampling-based planers are preferred when facing obstacle-cluttered environments. A recently developed algorithm called INSAT effectively combines graph search in the low-dimensional subspace and trajectory optimization in the full-dimensional space for global kinodynamic planning over long horizons. Although INSAT successfully reasoned about and solved complex planning problems, the numerous expensive calls to an optimizer resulted in large planning times, thereby limiting its practical use. Inspired by the recent work on edge-based parallel graph search, we present PINSAT, which introduces systematic parallelization in INSAT to achieve lower planning times and higher success rates, while maintaining significantly lower costs over relevant baselines. We demonstrate PINSAT by evaluating it on 6 DoF kinodynamic manipulation planning with obstacles.

\end{abstract}

\FloatBarrier
\section{Introduction}
\pdfoutput=1

Graph search-based planning algorithms like A* and its variants~\cite{hart1968formal, pohl1970heuristic, aine2016multi} enable robots to come up with well-reasoned long-horizon plans to achieve a given task objective~\cite{kusnur2021planning, mukherjee2021reactive}. They do so by searching over the graph that results from discretizing the state and action space. However, in robotics, several dynamically rich tasks require high-dimensional planning in the continuous space. For such domains, kinodynamic planning and trajectory optimization techniques have been developed to synthesize dynamically feasible trajectories. The existing kinodynamic algorithms achieve this by discretizing the action space to roll out trajectory primitives in the discrete or continuous state space. On the other hand, trajectory optimization techniques do not discretize the state or action space but suffer local minima, lack convergence guarantees in nonlinear settings, and struggle to reason over long horizons.

\begin{figure}
\centering
\includegraphics[width=\columnwidth]{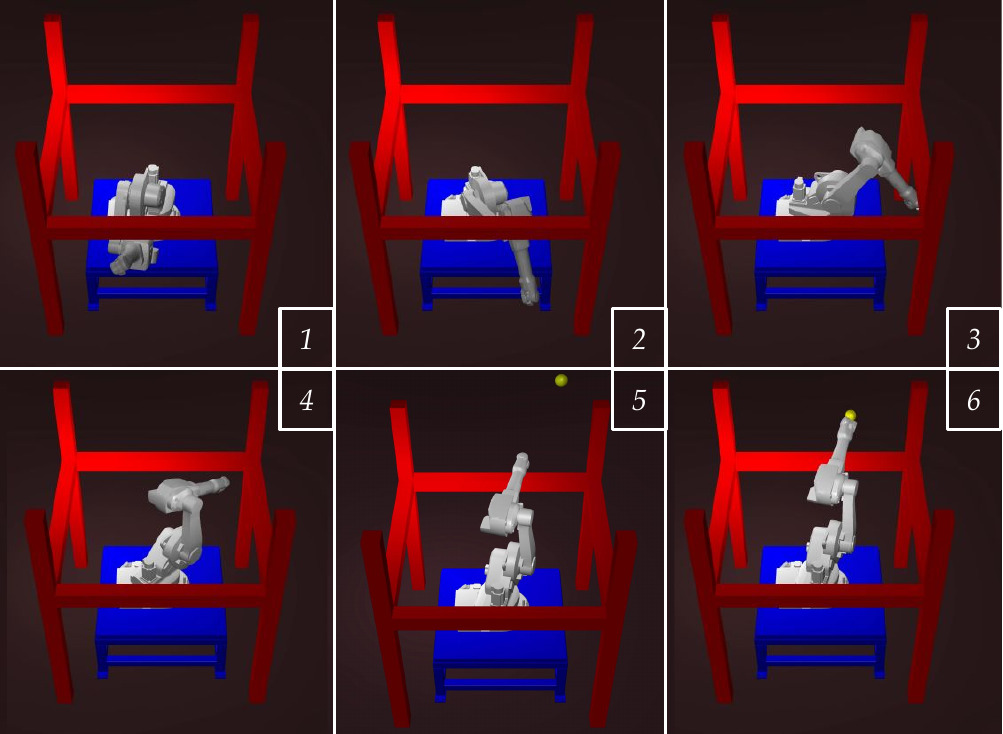}
\caption{ABB arm evading obstacles to block a ball thrown at it using kinodynamic motion produced by PINSAT.}
% \caption{Waypoints along a kinodynamic plan generated by PINSAT for an ABB arm to block a ball thrown at it.}
\label{fig:pinsat_abb}
\vspace{-0.5cm}
\end{figure}

An algorithm called INSAT \cite{insat,insatptc}, short for INterleaved Search And Trajectory optimization, bridges this gap for global kinodynamic planning. The idea behind INSAT was (a) to identify a low-dimensional manifold, (b) perform a search over a discrete graph embedded in this manifold, and (c) while searching the graph, utilize full-dimensional trajectory optimization to compute the cost of partial solutions found by the search. Thus every edge/action evaluation in the graph required solving at least one and potentially several instances of trajectory optimization. Though INSAT used dynamic programming to cleverly warm-start every instance of optimization with a good approximate solution, the frequent calls to the optimizer limited its practical usage.

% As a result, the search over the low-D graph decides what trajectory optimizations to run and with what seeds, while the cost of solution from the trajectory optimization drives the search in the lower-D graph until a feasible full-D trajectory from start to goal is found. 

For domains where action evaluation is expensive, a parallelized planning algorithm \epase (Edge-based Parallel A* for Slow Evaluations) was developed~\cite{mukherjee2022epase} that changes the basic unit of the search from state expansions to edge expansions. This decouples the evaluation of edges from the expansion of their common parent state, giving the search the flexibility to figure out which edges need to be evaluated to solve the planning problem. In this work, we employ the parallelization technique of \epase to parallelize the expensive trajectory optimization step in \insat. The resulting algorithm termed \pinsat: Parallel Interleaving of Graph Search and Trajectory Optimization, can compute dynamically feasible plans while achieving close to real-time performance. The key motivation for developing PINSAT is that if INSAT can effectively solve long-horizon, dynamically rich planning tasks, then using the ideas of \epase parallelization to expedite slow edge optimizations in INSAT will result in a strictly better algorithm. We empirically demonstrate that this is indeed the case in PINSAT.

\FloatBarrier
\section{Related Work}
\subsubsection{Kinodynamic Motion Planning} 
Kinodynamic planning has been a subject of extensive research, driven by the need for robots and autonomous systems to navigate complex environments while considering their dynamic constraints. This dynamic feasibility is typically achieved using four different schemes namely (i) search-based, (ii) sampling-based, (iii) optimization-based, and (iv) hybrid methods.

\textbf{Search-based} kinodynamic planning involves precomputing a set of short, dynamically feasible trajectories called motion primitives that capture the robot's capabilities. Search algorithms like A* or its variants \cite{maxprimitives} are used to search over a graph wherein for any state its successor states are being computed by applying motion primitives and provide guarantees of optimality and completeness w.r.t. the chosen primitives. However, the choice and calculation of these motion primitives that are efficient can be challenging, particularly in high-dimensional systems. \textbf{Sampling-based} kinodynamic methods adapt and extend classic approaches such as Probabilistic Roadmaps (PRM) \cite{prm} and Rapidly Exploring Random Trees (RRT) \cite{rrt} to handle dynamic systems \cite{kinodynamic}. This is achieved using dynamically feasible rollouts with random control inputs or solutions of boundary value problems within the \textit{extend} operation. There are probabilistically complete and asymptotically optimal variants \cite{completerrt, asymprrt}, however, empirical convergence might be tricky. \textbf{Optimization-based} planning methods can generate high-quality trajectories that are dynamically feasible and do not suffer from the curse of dimensionality in search-based methods. They formulate the motion planning problem as an optimization problem \cite{optctr} with cost functions defined for trajectory length, time, or energy consumption. After transcribing into a finite-dimensional optimization, these methods rely on the gradients of the cost function and dynamics and employ numerical optimization algorithms to find locally optimal solutions \cite{trajopt, komo, ddp}. However, except for a small subset of systems (such as linear or flat systems), for most nonlinear systems these methods lack guarantees on completeness, optimality or convergence. \textbf{Hybrid} planning methods combine two or all of the schemes mentioned above. Search and sampling methods are combined in \cite{bit, ss1, ss2}, sampling and optimization methods are combined in \cite{rabit, so1, so2}, search, sampling, and optimization methods are combined in \cite{idba}. INSAT combined search and optimization methods and demonstrated its capability in several complex dynamical systems \cite{insat, insatptc}. 

\subsubsection{Parallel Search}
Several approaches parallelize sampling-based planning algorithms in which parallel processes cooperatively build a PRM~\cite{jacobs2012scalable} or an RRT~\cite{devaurs2011parallelizing, ichnowski2012parallel, jacobs2013scalable}  by sampling states in parallel. However, in many planning domains, sampling of states is not trivial. One such class of planning domains is simulator-in-the-loop planning, which uses a physics simulator to generate successors~\cite{liang2021search}. Unless the state space is simple, such that the sampling distribution can be scripted, there is no principled way to sample meaningful states that can be realized in simulation.

Search-based methods like A* can be parallelized by generating successors concurrently during state expansion but are limited by the domain's branching factor. Alternatively, approaches like~\cite{irani1986parallel, evett1995massively, zhou2015massively} parallelize state expansions, allowing re-expansions to handle premature expansions before minimal cost, but may face exponential re-expansions, especially with weighted heuristics. On the contrary, \pase~\cite{phillips2014pa} parallelly expands states at most once without affecting solution quality bounds. However, \pase is inefficient in domains with costly edge evaluations as each thread sequentially evaluates the outgoing edges. To address this, \epase~\cite{mukherjee2022epase} improves \pase by parallelizing edge search. MPLP~\cite{mukherjee2022mplp} achieves faster planning by lazily running the search and asynchronously evaluating edges in parallel, assuming that successors can be generated without evaluating the edge. Some work focuses on parallelizing A* on GPUs~\cite{zhou2015massively, he2021efficient}, but their SIMD execution model limits applicability to domains with simple actions sharing the same code.

% To the best of our knowledge, there is very little previous work on parallel kinodynamic planning. They are all sampling-based and employ a straightforward parallelization to run the \textit{steer} operation in batch either in CPU for steering with boundary value solvers \cite{pmpc} or GPU for steering with neural networks in batch \cite{pnn}. Consequently, they inherit the same limitations of sampling-based planners in that sampling of states can be nontrivial. Search-based methods overcome this limitation while systematically exploring the space and INSAT owes its success to this fact. The key motivation for developing PINSAT is that if INSAT was able to effectively solve long-horizon, dynamically rich planning tasks, then using the ideas of edge parallelization from \epase to expedite slow edge optimizations in INSAT will produce a strictly better algorithm. We empirically show that this is indeed the case in PINSAT. 

To the best of our knowledge, there is very little previous work on parallel kinodynamic planning. All existing approaches are sampling-based and employ a straightforward parallelization to execute the \textit{steer} operation in batch—either on CPU for steering with boundary value solvers \cite{pmpc} or on the GPU for steering with neural networks in batch \cite{pnn}. Consequently, they inherit the same limitations as sampling-based planners, where the sampling of states can be nontrivial. Search-based methods overcome this limitation by systematically exploring the space, and INSAT \& PINSAT owe their success to this characteristic.

 % \epase is efficient in domains where all action evaluations are expensive, but not in domains where the actions space comprises a mix of cheap and expensive to compute actions. 

\FloatBarrier
\section{Problem Formulation}
\pdfoutput=1

% Let a finite graph $\graph = (\Vertices, \Edges)$ be defined as a set of vertices \Vertices and directed edges \Edges. Each vertex $\vertex \in \Vertices$ represents a state \state in the state space of the domain \States. An edge $\ed \in \Edges$ connecting two vertices $\vertex_1$ and $\vertex_2$ in the graph represents an action $\ac \in \Aset$ that takes the agent from corresponding states $\state_1$ to $\state_2$. In this work, we assume that all actions are deterministic. Hence an edge \ed can be represented as a pair \edge, where \state is the state at which action \ac is executed. For an edge \ed, we will refer to the corresponding state and action as $\ed.\state$ and $\ed.\ac$ respectively. 

% In addition, we will use the following notations:
% \begin{itemize}
%     \item \startstate is the start state and \goalreg is the goal region.
%     \item $\cost:\Edges \rightarrow [0,\infty]$ is the cost associated with an edge.
%     \item $\gval(\state)$ or g-value is the cost of the best path to \state from \startstate found by the algorithm so far.
%     \item $\hval(\state)$ is a consistent and therefore admissible heuristic~\cite{russell2010artificial}. It never overestimates the cost to the goal.
% \end{itemize}

Let the $n$-dimensional state space of the robot be denoted by $\statespace \subseteq \mathbb{R}^n$. Let $\obsspace \subset \statespace$ be the obstacle space and $\freespace = \statespace \setminus \obsspace$ be the free motion planning space. For kinodynamic motion planning, we should reason about and satisfy constraints on the derivatives of position such as velocity, acceleration, etc. This boils down to including those derivatives as a part of the planning state and can quickly lead to a large intractable planning space. To cope with this let us consider a low-dimensional (low-D) space $\lowdspace$ and an auxiliary space $\auxspace$ whose Cartesian product forms the full-dimensional (full-D) space $\statespace = \auxspace \times \lowdspace$. Let the low-D space be such that $\obsspace\subset\lowdspace$. Consider a finite graph $\graph = (\Vertices, \Edges)$ defined as a set of vertices \Vertices and directed edges \Edges embedded in this low-D space $\lowdspace$. Each vertex $\vertex \in \Vertices$ represents a state $\state \in \lowdspace$. An edge $\ed \in \Edges$ connecting two vertices $\vertex_1$ and $\vertex_2$ in the graph represents an action $\ac \in \Aset$ that takes the agent from corresponding states $\state_1$ to $\state_2$. In this work, we assume that all actions are deterministic. Hence an edge \ed can be represented as a pair \edge, where \state is the state at which action \ac is executed. For an edge \ed, we will refer to the corresponding state and action as $\ed.\state$ and $\ed.\ac$ respectively. So given a start state $\state^S$ and a goal region $\goalreg$, the kinodynamic motion planning problem can be cast as the following optimization problem
\begin{subequations}
\begin{alignat}{2}
% & \traj_{\state^S\state^G}(t), T_{\state^S\state^G} = \argmin_{\traj(t), T}\quad w_1T + w_2\int_0^T \traj(t)dt \\ 
\min_{\traj(t), T} \quad & w_1T + w_2\int_0^T \traj(t)dt \label{eq:obj1}\\ 
\textrm{s.t.} \quad & 0 < t_{\text{min}} \leq T \leq t_{\text{max}} \label{eq:obj2}\\
& \traj(0) = \state^S,  \traj(T) = \state^G \in \goalreg \label{eq:obj3}\\
% & |\trajdot(t)| \preceq \trajdot_{\text{lim}}, |\trajddot(t)| \preceq \trajddot_{\text{lim}}, |\trajdddot(t)| \preceq \trajdddot_{\text{lim}}; \forall t \\
& |\traj^{(j)}(t)| \preceq \traj^{(j)}_{\text{lim}} > 0; \forall t; j=\{1,2,\ldots,\gamma\} \label{eq:obj4}\\
& \traj(t) \in \freespace; \forall t \label{eq:obj5}\\
& \traj(t) \in \mathcal{C}^\gamma(\mathbb{R}); \forall t \label{eq:obj6}
\end{alignat}
\label{eq:obj}
\end{subequations}
where $\traj : [0, T] \rightarrow \lowdspace$ is a spatial trajectory that is $\gamma$-times differentiable and therefore continuous, $w_1$ and $w_2$ trade-off the relative importance between path length and duration, and $\traj^{(j)}(t) =\frac{d^j\traj(t)}{dt^j}$ is the $j$-th derivative of $\traj(t)$. The Eq. \ref{eq:obj2} is the constraint on the maximum duration of the trajectory, Eq. \ref{eq:obj3} is the boundary start and goal conditions, Eq. \ref{eq:obj4} is the element-wise system limit on the derivatives of the positional trajectory, Eq. \ref{eq:obj5} requires the trajectory to lie in $\freespace$ and Eq. \ref{eq:obj6} imposes the degree of continuity on the trajectory. The minimizer of the Eq. \ref{eq:obj} is a trajectory $\traj_{\textbf{x}^S\textbf{x}^G}(t)$\footnote{The argument $t$ of the $\traj(t)$ may be dropped for brevity.} that connects start state $\state^S$ to the goal region $\goalreg$. There is a computational budget of \numthreads threads available, which can run in parallel. 

% Consider an invertible many-to-one mapping $\boldsymbol{\lambda}: \mathcal{X} \longrightarrow \mathcal{X}_L$ that projects a full dimensional state $\textbf{x} = [\textbf{q}, \dot{\textbf{q}}] \in \mathcal{X}$ into the low-dimensional space $\mathcal{X}_L$. So $\textbf{x}_L = \boldsymbol{\lambda}(\textbf{x})$. Then $\boldsymbol{\lambda}^{-1}: \mathcal{X}_L \longrightarrow \mathcal{X}$ is an one-to-many inverse mapping of $\boldsymbol{\lambda}$ that lifts a low dimensional state $\textbf{x}_L \in \mathcal{X}_L$ to any possible full-dimensional state $\textbf{x} \in \mathcal{X}$. So $\textbf{x} = \boldsymbol{\lambda}^{-1}(\textbf{x}_L)$. $\boldsymbol{\phi}_{\textbf{x}^\prime\textbf{x}^{\prime\prime}}(t)$ denotes a time $t$ parameterized full-D trajectory from $\boldsymbol{\lambda}^{-1}(\textbf{x}^\prime_L)$ to $\boldsymbol{\lambda}^{-1}(\textbf{x}^{\prime\prime}_L)$. The argument $t$ may be dropped for brevity.

% A path \plan is defined by an ordered sequence of edges $\ed_{i=1}^N = \edge_{i=1}^N$, the cost of which is denoted as $\cost(\plan) = \sum_{i=1}^N \cost(\ed_i)$. The objective is to find a path \plan from $\state_0$ to a state in the goal region \goalreg with the optimal cost \costopt. There is a computational budget of \numthreads threads available, which can run in parallel. 

\FloatBarrier
\section{Background}
\label{sec:bg}
\pdfoutput=1

% In robotics and motion planning, kinodynamic planning is a class of problems for which velocity, acceleration, and force/torque bounds must be satisfied, together with kinematic constraints such as avoiding obstacles.

% However, controllers in fully actuated systems like the vast majority of commercial manipulators do not require a fully dynamically feasible trajectory to track them accurately. In such systems, even a velocity controller can track trajectories generated with smooth splines \cite{planningsp} at high accuracy.

In kinodynamic planning, a great deal of focus is on the integration accuracy of the equations of motion to obtain feasible solutions to the dynamic constraints. However, fully controllable systems, like the vast majority of commercial manipulators, can accurately track trajectories generated with smooth splines \cite{planningsp} even if they are mildly inconsistent with the system's nonlinear dynamics. Even for many other dynamical systems, parameterizing trajectories with polynomials achieves near-dynamic feasibility and proves to be highly effective. Popular techniques like direct collocation and pseudo-spectral methods parameterize trajectories with polynomials at some level. In this work, we use a piecewise polynomial widely used recently for motion planning \cite{manipspline, planningbsp} called B-spline to represent trajectories. We will first provide some background on B-splines and highlight a few nice properties that enable us to satisfy dynamics and guarantee completeness.

\subsection{B-Splines}
B-splines are smooth and continuous piecewise polynomial functions made of finitely many basis polynomials called B-spline bases. A $k$-th degree B-spline basis with $m$ control points can be calculated using the Cox-de Boor recursion formula \cite{patrikalakis} as
\begin{equation}
N_{i, k}(t)=\frac{t-t_i}{t_{i+k-1}-t_i} N_{i, k-1}(t)+\frac{t_{i+k}-t}{t_{i+k}-t_{i+1}} N_{i+1, k-1}(t)
\label{eq:bsb}
\end{equation}
where $i=0, \ldots, m$, $\frac{t-t_i}{t_{i+k-1}-t_i}$ and $\frac{t_{i+k}-t}{t_{i+k}-t_{i+1}}$ are the interpolating coefficients between $t_i$ and $t_{i+k}$. For $k=0$, $N_{i,0} = 1$ if $t_i \leq t<t_{i+1}$ and 0 otherwise. 
% is given by
% \begin{equation}
% N_{i, 0}(t)=\left\{\begin{array}{ll}
% 1 & \text { if } t_i \leq t<t_{i+1} \\
% 0 & \text { otherwise }
% \end{array}\right.
% \end{equation}

% Let us define a non-decreasing knot vector $\textbf{T}$ 
% \begin{equation}
% \textbf{T} = \{\underbrace{t_0, \ldots, t_0}_{k-\text{points}}, t_{k+1}, \ldots, t_{m}, \underbrace{t_f, \ldots, t_f}_{k-\text{points}}\}    
% \label{eq:knotvec}
% \end{equation} 
% and the set of control points $\textbf{P}$
% \begin{equation}
%     \textbf{P} = \{\textbf{p}_0, \textbf{p}_1, \ldots, \textbf{p}_m\}
%     \label{eq:ctrlvec}
% \end{equation}

Let us define a non-decreasing knot vector $\textbf{T}$ and the set of control points $\textbf{P}$ called de Boor points
\begin{subequations}
\begin{align}
\textbf{T} = \{\underbrace{t_0, \ldots, t_0}_{k-\text{points}}, t_{k+1}, \ldots, t_{m}, \underbrace{t_f, \ldots, t_f}_{k-\text{points}}\}\label{eq:knotvec} \\
\textbf{P} = \{\textbf{p}_0, \textbf{p}_1, \ldots, \textbf{p}_m\}\label{eq:ctrlvec}
\end{align}
\end{subequations}

% A knot vector representation with a multiplicity of $k$ at the endpoints is chosen to guarantee $\mathcal{C}^{n-k-1}$ continuity [cite].
where $\textbf{p}_i \in \mathbb{R}^n, i=0,\ldots, m$. Then our B-spline trajectory $\traj(t)$ can be uniquely determined by $\textbf{T}$, $\textbf{P}$ and  the degree of the polynomial $k$
\begin{equation}
    \traj(t) = \sum_{i=0}^m \textbf{p}_i N_{i, k}(t) 
    \label{eq:bs}
\end{equation}
\begin{remark}
    The above B-spline trajectory $\traj(t)$ is guaranteed to entirely lie inside the convex hull of the active control points $\textbf{P}$ (the control points whose bases are not zero).
\end{remark}

We need to evaluate the derivatives of $\traj(t)$ to satisfy the limits on joint velocity and acceleration. For this, the $j$-th derivative of $N_{i,k}(t)$ is given as 
\begin{equation}
N_{i, k}^{(j)}(t)=j\left(\frac{N_{i, k-1}^{(j-1)}(t)}{t_{i+k}-t_i}-\frac{N_{i+1, k-1}^{(j-1)}(t)}{t_{i+k+1}-t_{i+1}}\right)    
\end{equation}
\begin{remark}
    The derivative of a B-spline is also a B-spline of one less degree and it is given by 
\begin{equation}
    \traj^{(j)}(t) = \sum_{i=0}^n \textbf{p}_i N_{i, k}^{(j)}(t)
    \label{eqn:bsderiv1}
\end{equation}
\end{remark}

% \FloatBarrier
% \section{B-Spline Optimization}
% \label{sec:methods}
% \input{06bsplineopt}

\FloatBarrier
\section{Approach}
\label{sec:methods}

% This section explains the operation we perform to lift the low-D edges $e\in\Edges$ in the graph $G$ to full-D. Specifically, given an edge, we describe our trajectory optimization setup to incorporate the cost and constraints in $\statespace$ in addition to those imposed by the edge living in $\lowdspace$. 

PINSAT interleaves \textit{parallelized} discrete graph search in low-D with trajectory optimization in full-D to combine the benefits of the former's ability to search non-convex spaces and solve combinatorial parts of the problem and the latter's ability to obtain a locally optimal solution not constrained to discretization. In the following subsection, we will justify the choice of the low-D search algorithm. Subsequently, we will explain how an edge in the graph in $\lowdspace$ is lifted to $\statespace$ using B-spline optimization and the necessary modifications to the low-D search to make it compatible with PINSAT.

\subsection{Low Dimensional Graph Search}
The low-D search runs \weas~\cite{mukherjee2022epase} with variations that allow it to be interleaved with full-D trajectory optimization. Using \weas instead of \wastar (like in INSAT) for low-D search provides a systematic framework to parallelize the optimization of expensive trajectory with its parallelized variant \wepase. This is because, unlike in \wastar, where the basic operation of the search loop is the expansion of a state, in \weas, the basic operation is \emph{expansion of an edge}. 

\subsubsection{\weas}
\label{sec:weas}
In \weas, the open list (\open) is a priority queue of edges (not states like in \wastar) that the search has generated but not expanded with the edge with the smallest priority in front of the queue. The priority of an edge is $\fval\left(\edge\right) = \gval(\state) + \hval(\state)$. Expansion of an edge \edge involves evaluating the edge to generate the successor $\state'$ and adding/updating (but not evaluating) the edges originating from $\state'$ into \open with the same priority of $\gval(\state') + \hval(\state')$. Henceforth, whenever $\gval(\state')$ changes, the positions of all of the outgoing edges from $\state'$ need to be updated in \open. To avoid this, \epase replaces all the outgoing edges from $\state'$ by a single \textit{dummy} edge $(\state', \dac)$, where \pac stands for a dummy action until the dummy edge is expanded. Every time $\gval(\state')$ changes, only the dummy edge's position has to be updated. When the dummy edge $(\state', \dac)$ is expanded, it is replaced by the outgoing real edges from $\state'$ in \open. The real edges are expanded when they are popped from \open by an edge expansion thread. This decoupling of the expansion of the outgoing edges from the expansion of their common parent is what enables its asynchronous parallelization.

\subsection{B-Spline Optimization}
This subsection explains the abstract optimization over the length and duration of B-splines that satisfy the obstacle, duration, derivative, and boundary constraints. In the later section, we describe how it is utilized within \pinsat's parallel low-D graph search. 

\subsubsection{Decoupling Optimization over Trajectory and its Duration}
It is easy to observe from Eq. \ref{eq:obj} that there is a coupling between the trajectory $\traj(t)$ and its duration $T$. For simultaneous optimization over the $\traj$ and $T$, we decouple them by introducing a trajectory coordinate $u \in [0, 1]$ and a mapping to convert this coordinate to the actual time, $t=\timemap(u)$ \cite{planningsp}. Let the mapping $\timemap$ be monotonically increasing with $\timemap(0) = 0$ and $\timemap(1)=T$ and $\traj(t) = \traj(\alpha(u)) = (\traj\circ\alpha)(u) = \utraj(u)$. The derivatives of the trajectory are
\begin{equation}
    \trajdot(t) = \frac{d\traj(\alpha(u))}{d\alpha}\frac{d\alpha(u)}{du}\frac{du}{dt} = \frac{d\utraj(u)}{du}\frac{du}{dt}
    \label{eq:trajderiv}
\end{equation}
Dropping the arguments for simplicity
% \begin{multline*}
%     \trajddot(t) = \frac{d}{dt}\left(\frac{d\traj}{d\alpha}\right) \frac{d\alpha}{du}\frac{du}{dt} \\ + \frac{d\traj}{d\alpha}\frac{d}{dt}\left(\frac{d\alpha}{du}\right)\frac{du}{dt} + \cancelto{0}{\frac{d\traj}{d\alpha}\frac{d\alpha}{du}\frac{d}{dt}\left(\frac{du}{dt}\right)}
% \end{multline*}

% \begin{multline}
%     % \trajddot(t) = \frac{d^2\traj}{du^2}\left(\frac{du}{dt}\right)^2 + \frac{d\traj}{d\alpha}\frac{d}{dt}\left(\frac{d\alpha}{du}\right)\frac{du}{dt} \\ = \frac{d\utraj^2}{du^2}\left(\frac{du}{dt}\right)^2 + \frac{d\utraj}{ds}\frac{d^2u}{dt^2}
% \end{multline}
\begin{equation}
    \trajddot(t) = \frac{d\utraj^2}{du^2}\left(\frac{du}{dt}\right)^2 + \frac{d\utraj}{ds}\frac{d^2u}{dt^2}
    \label{eq:trajdderiv}    
\end{equation}
Similarly
\begin{equation}
    \trajdddot(t) = \frac{d^3\utraj}{du^3}\left(\frac{du}{dt}\right)^3 + 3\frac{d^2\utraj}{du^2}\frac{d^2u}{dt^2}\frac{du}{dt} + \frac{d\utraj}{du}\frac{d^3u}{dt^3}    
    \label{eq:trajddderiv}
\end{equation}

\pdfoutput=1

\begin{algorithm}[H]
% \caption{\label{alg:pinsat} \pinsat: Planning Loop}
\caption{\pinsat: Planning Loop}
\label{alg:pinsat}
\begin{footnotesize}
\begin{algorithmic}[1]
\STATE $\Aset\gets\text{ action space }$, $\numthreads \gets$ number of threads, $\graph \gets \emptyset$
\STATE $\startstate\gets\text{ start state }$, $\goalreg\gets\text{ goal region}$, $terminate \gets \text{False}$
\PROCEDURE{Plan}{}{}
    \STATE $\forall\state\in\graph$,~$\state.\gval\gets\infty$,~$\numexpanded(s)=0$ \label{alg:pinsat/init1}
    \STATE $\startstate.\gval\gets0$
    \STATE insert $(\startstate, \dac)$ in \open \COMMENT{Dummy edge from \startstate} \label{alg:pinsat/init2}
    \STATE LOCK
    \WHILE{$\textbf{not } terminate$}
        \IF{$\open=\emptyset\textbf{ and }\be=\emptyset$} \label{alg:pinsat/be}
            \STATE $terminate = \text{True}$
            \STATE UNLOCK
            \RETURN $\emptyset$
        \ENDIF        
        \STATE $\edge \gets \open.min()$
        \label{alg:pinsat/open_pop}
        \IF{such an edge does not exist}
            \STATE UNLOCK
            \STATE wait until \open or \be change
            \label{alg:pinsat/wait}
            \STATE LOCK
            \STATE continue
        \ENDIF
        \IF{$\state \in \goalreg$}
            \STATE $terminate = \text{True}$
            \STATE UNLOCK
            \RETURN $\state.traj$
            \label{alg:pinsat/return_traj}
        \ELSE
            \STATE UNLOCK
            \WHILE{\edge has not been assigned a thread}
                \FOR{$i=1:\numthreads$}
                    \IF{thread $i$ is available}
                        \IF{thread $i$ has not been spawned}
                            \STATE Spawn $\textsc{EdgeExpandThread}(i)$
                            \label{alg:pinsat/spawn}
                        \ENDIF
                        \STATE Assign \edge to thread $i$
                        \label{alg:pinsat/assign_edge}
                    \ENDIF
                \ENDFOR
            \ENDWHILE
            \STATE LOCK
        \ENDIF
    \ENDWHILE 
    \STATE $terminate = \text{True}$
    \STATE UNLOCK
\ENDPROCEDURE
\end{algorithmic}
\end{footnotesize}
\end{algorithm}

% Now we can apply the decoupling introduced above to the optimization in Eq. \ref{eq:obj} with the new trajectory coordinate $u$ and the map $\utraj$ introduced above. 
\subsubsection{Transcription for Optimizing B-splines} The derivative of the B-spline curve $\traj^{(j)}(t)$ in terms of the derivative of the B-spline basis $N_{i, k}^{(j)}(t)$ is given in Eq. \ref{eqn:bsderiv1}. Using this relation the derivative of the B-spline curve with the trajectory coordinate $\utraj^{(j)}(u)$ in terms of the derivative's control points $\textbf{p}_i^{(j)}$ can be derived as \cite{nurbs}
\begin{equation}
    \utraj^{(j)}(u) = \sum_{i=0}^{n-j} \textbf{p}_i^{(j)} N_{i, k-j}(u)
    \label{eq:bsdervi2}
\end{equation}
where $\textbf{p}_i^{(j)}$ is given as
\begin{equation}
\textbf{p}_i^{(j)}= \begin{cases}\textbf{p}_i & j=0 \\ \frac{k-j+1}{u_{i+k+1}-u_{i+j}}(\textbf{p}_{i+1}^{(j-1)}-\textbf{p}_i^{(j-1)}) & j>0\end{cases}
\end{equation}

% In this work, we use the relation $t = \alpha(u) = T*u$ between time and the trajectory coordinate. Under this relation, we can see that
% \begin{align}
%     \frac{du}{dt} = \frac{1}{T} \quad \text{and} \quad \frac{d^2u}{dt^2} = 0
% \end{align}
In this work, we use $t = \alpha(u) = Tu$. So $\frac{du}{dt} = \frac{1}{T}$ and $\frac{d^2u}{dt^2} = \frac{d^3u}{dt^3} = 0$ and Eq. \ref{eq:trajderiv} - Eq. \ref{eq:trajddderiv} becomes 
\begin{align}
    % \trajdot(t) = \frac{1}{T}\frac{d\utraj}{du} \quad \trajddot(t) = \frac{1}{T^2}\frac{d^2\utraj}{du^2} \quad \trajdddot(t) = \frac{1}{T^3}\frac{d^3\utraj}{du^3}
    \traj^{(j)}(t) = \frac{1}{T^j}\frac{d^j\utraj}{du^j} \quad \text{where} \ j=\{1,2,3\}
    \label{eq:ttou}
\end{align}

\pdfoutput=1

\begin{algorithm}[H]
% \caption{\label{alg:pinsat_expand} \pinsat: Edge Expansion}
\caption{\pinsat: Edge Expansion}
\label{alg:pinsat_expand} 
\begin{footnotesize}
\begin{algorithmic}[1]
\PROCEDURE{EdgeExpandThread}{$i$}
    \WHILE{$\textbf{not } terminate$}
        \IF{thread $i$ has been assigned an edge \edge}
            \STATE $\textsc{Expand}\left(\edge\right)$
        \ENDIF
    \ENDWHILE
\ENDPROCEDURE
\PROCEDURE{Expand}{$\edge$}
    \STATE LOCK
    \IF{$\ac = \pac$} \label{alg:pinsat_expand:dummy1}
        \STATE insert \state in \be
        \label{alg:pinsat_expand/add_be}
        % \STATE insert \STATE in \closed
        \FOR{$\ac \in \Aset$}
            \STATE $\fval\left(\edge\right) = \gval(\state) + \hval(\state)$
            \STATE insert \edge in \open with $\fval\left(\edge\right)$
        \ENDFOR \label{alg:pinsat_expand:dummy2}
    \ELSE \label{alg:pinsat_expand/real1}
        \STATE UNLOCK
        \STATE $\state', \cost\left(\edge\right)  \gets \textsc{GenerateSuccessor}
        \left(\edge\right)$
        \label{alg:pinsat_expand/evaluate}
        \STATE LOCK
        \IF{$\state' \notin \closed\cup\be$} \label{alg:pinsat_expand/clorbe}
            \STATE $\traj_{\startstate\nextstate} = \textsc{GenerateTrajectory}(\state, \nextstate)$
            \IF{$\gval(\state')>c(\traj(\startstate\nextstate))$}
                \STATE $\gval(\state') = c(\traj(\startstate\nextstate))$
                \STATE $\state'.parent = \state$ 
                \STATE $\state'.traj = \traj_{\startstate\nextstate}$
                \label{alg:pinsat_expand/set_traj}
                \STATE $\fval\left(\pedgenext\right) = \gval(\state') + \hval(\state')$
                \STATE insert/update \pedgenext in \open with $\fval\left(\pedgenext\right)$
            \ENDIF
        \ENDIF
        \STATE $\numexpanded(\state)+=1$
        \IF{$\numexpanded(\state) = |\Aset|$}
            \STATE remove \state from \be
            \label{alg:pinsat_expand/remove_be}
            \STATE insert \state in \closed
            \label{alg:pinsat_expand/add_closed}
        \ENDIF 
    \ENDIF\label{alg:pinsat_expand/real2}
    \STATE UNLOCK
\ENDPROCEDURE
\end{algorithmic}
\end{footnotesize}
\end{algorithm}

Now we can apply the above decoupling to the optimization in Eq. \ref{eq:obj} with the new trajectory coordinate $u$ and the map $\utraj$. Given two boundary states $\state^\prime, \state^{\prime\prime}$, Eq. \ref{eq:obj} can be transcribed into a parameter optimization problem over the control points of the B-spline trajectory $\textbf{P} = \{\textbf{p}_0, \ldots, \textbf{p}_m\}$ and the duration of the trajectory $T$ using Eq. \ref{eq:ttou} as
\begin{subequations}
\begin{alignat}{2}
% & \traj_{\state^S\state^G}(t), T_{\state^S\state^G} = \argmin_{\traj(t), T}\quad w_1T + w_2\int_0^T \traj(t)dt \\ 
& \min_{\textbf{p}_0, \ldots, \textbf{p}_m, T} \quad w_1T + w_2\sum_{i=0}^{m-1} \norm{\textbf{p}_{i+1} - \textbf{p}_i}_2 \label{eq:uobj1}\\ 
\textrm{s.t.} \ \ & 0 < t_{\text{min}} \leq T \leq t_{\text{max}} \label{eq:uobj2}\\
% & \textbf{p}_0 = \state^S,  \textbf{p}_m = \state^G \in \goalreg \label{eq:uobj3}\\
& \textbf{p}_0 = \state^\prime,  \textbf{p}_m = \state^{\prime\prime} \label{eq:uobj3}\\
& |\textbf{p}_i^{(j)}| \preceq T^j\traj^{(j)}_{\text{lim}}; i=\{0,\ldots,m\}, j=\{1,2,3\} \label{eq:uobj4}\\
% & \utraj_{\textbf{p}_i}(u) \in \freespace; u\in[0,1], i=\{0,\ldots,m\} \label{eq:uobj5}\\
% & \utraj_{\textbf{p}_i}(u) \in \mathcal{C}^\gamma(\mathbb{R}); u\in[0,1], i=\{0,\ldots,m\} \label{eq:uobj6}
& \utraj_{\textbf{P},T}(u) \in \freespace; u\in[0,1] \label{eq:uobj5}
% & \utraj_{\textbf{P}}(u) \in \mathcal{C}^\gamma(\mathbb{R}); u\in[0,1]\label{eq:uobj6}
\end{alignat}
\label{eq:uobj}
\end{subequations}
where the constraint Eq. \ref{eq:uobj4} comes from Eq. \ref{eq:ttou} and $\utraj_{\textbf{P},T}$ is the trajectory reconstructed during every iteration using $\textbf{P}$ and $T$ as explained below. From the properties of B-splines \cite{boor}, we know that the B-spline curve is entirely contained inside the convex hull of its control points. We leverage this property to formulate Eq. \ref{eq:uobj1} and Eq. \ref{eq:uobj4}. So it is enough to just enforce the trajectory derivative constraints at the control points (Eq. \ref{eq:uobj4}) instead of checking if the curve satisfies them for all $u$ (and therefore for all $t$). 
\begin{remark}
The constraint on the control points in Eq. \ref{eq:uobj4} is a sufficient but not necessary to enforce the limits on derivatives. In other words, it is possible to have control points outside the limits $\pm T^j\traj^{(j)}_{\text{lim}}$ and still have the curve satisfy these limits.
\label{rem:suff}
\end{remark}
Note that Eq. \ref{eq:uobj} is a nonlinear program (NLP) for $j>1$ as it has quadratic (or cubic, quintic, etc) constraints in the decision variable $T$ (Eq. \ref{eq:uobj4}). However, the following relaxation obtained by fixing the duration to $t_\text{max}$ is convex if $\freespace$ is convex
\begin{subequations}
\begin{alignat}{2}
&\min_{\textbf{p}_0, \ldots, \textbf{p}_m} \quad w\sum_{i=0}^{m-1} \norm{\textbf{p}_{i+1} - \textbf{p}_i}_2^2 \label{eq:tobj1}\\ 
\textrm{s.t.} \ &  \textbf{p}_0 = \state^\prime,  \textbf{p}_m = \state^{\prime\prime} \label{eq:tobj3}\\
& |\textbf{p}_i^{(j)}| \preceq t_{\text{max}}^j\traj^{(j)}_{\text{lim}}; i=\{0,\ldots,m\}, j=\{1,2,3\} \label{eq:tobj4}\\
& \traj_{\textbf{P},t_{\text{max}}}(t) \in \freespace; t\in[0,t_{\text{max}}] \label{eq:tobj5}
\end{alignat}
\label{eq:tobj}
\end{subequations}
% If $\freespace$ is convex, then the above problem can be solved to global optimality provided $t_{\text{max}}$ is large enough to ensure a nonempty feasible region.

\subsubsection{B-spline Trajectory Reconstruction} 
\label{sec:recon}
Once the control point vector $\textbf{P}$ and the duration of trajectory $T$ are found by solving Eq. \ref{eq:uobj}, the trajectory and its derivatives can be computed given the choice of B-spline basis $N_{i, k}(t)$ (Eq. \ref{eq:bsb}) and the knot vector $\textbf{T}$ (Eq. \ref{eq:knotvec}) using Eq. \ref{eq:bs} and Eq. \ref{eq:bsdervi2}. We also use this reconstruction to validate Eq. \ref{eq:uobj5} within every iteration of the solver routine for obstacle avoidance, early termination and constraint satisfaction.

\begin{figure*}[ht]
\centering
\includegraphics[width=\textwidth]{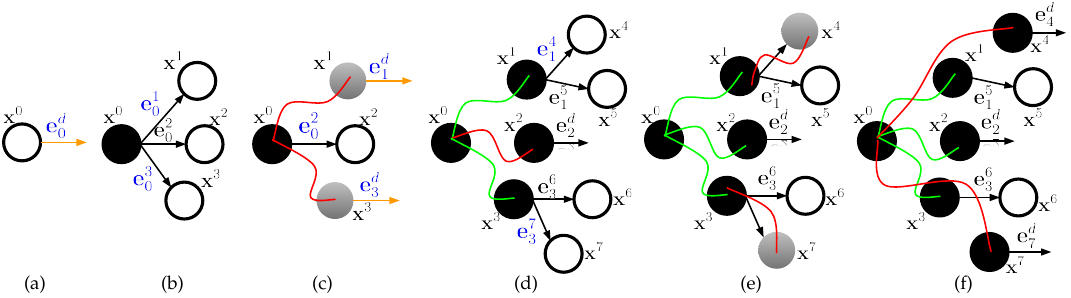}
\caption{\small Graphical illustration of PINSAT. (Fig. a, b) The dummy edge $e^d_0$ from $\state^0$ is expanded to get real low-D edges $[e^1_0, e^2_0, e^3_0]$ but not yet evaluated with the optimizer in full-D to generate the state successor (denoted with hollow node). (c) The \epase architecture in PINSAT then evaluates multiple edges in parallel (red curves). Once evaluated the outgoing nodes $\state^1\&\state^3$ are represented with dummy edges $e^d_1$ \& $e^d_3$ and inserted in OPEN. This node is shown with a gray gradient to represent the underspecified full-D state for the optimization to figure out. In Fig. (d), we highlight the asynchronous execution of PINSAT. Here the dummy edges $e^d_1$ \& $e^d_3$ are expanded into real edges and the $e^2_0$ is lifted to full-D by running the optimizer. Similarly, in Fig. (e), more edges are concurrently lifted to full-D. Fig. (f) denotes how the short incremental trajectory generated in Fig. (e) is reused for warm-starting the optimization and generating the trajectories from $\state^S$. This is a powerful step that dramatically increases the speed of convergence. Note that once the trajectory from $\state^S$ is generated, it replaces the low-D edge and helps drive the graph search into more informative directions.}
\label{fig:pinsat}
\vspace{-0.4cm}
\end{figure*}
% PINSAT uses the concept of edge expansion in graph search along with systematic parallelization introduced in e-PA*SE without sacrificing its algorithmic correctness in the trajectory optimization step of PINSAT to generate global kinodynamic motion plans in real-time.

\subsection{PINSAT: Parallelized Interleaving of Search And Trajectory Optimization}

% Though we generate the graph implicitly 
% The pseudocode of PINSAT is given in Alg. \ref{alg:pinsat} along with a graphical illustration in Fig. \ref{fig:pinsat}. In graph search terminology, edge computation refers to evaluating the cost and representation of an edge that is part of the underlying graph. The optimizing step of rewiring to a better parent node found in many search algorithms is restricted to the edges that are part of the graph. In contrast, a key difference in INSAT is that it dynamically generates edges between nodes that are not part of the low-D graph when searching for a full-D trajectory to a node. INSAT searches for this dynamic edge over the set of ancestor nodes leading to the node picked for expansion. Consequently, to make \weas and \epase compatible with INSAT and PINSAT, we search over the set of ancestor edges leading up to the edge picked for expansion (Alg. \ref{alg:pinsat_opt}, line \ref{alg:pinsatopt/anc}). 

The pseudocode of PINSAT is given in Alg. \ref{alg:pinsat} along with a graphical illustration in Fig. \ref{fig:pinsat}. The caption in Fig. \ref{fig:pinsat} is intentionally detailed and not repeated in this section. In \wepase, to maintain bounded suboptimality, an edge can only be expanded if it is \emph{independent} of all edges ahead of it in \open and the edges currently being expanded, i.e. in set \be (Alg. \ref{alg:pinsat}, line \ref{alg:pinsat/be})~\cite{mukherjee2022epase}. However, since \insat and, therefore, \pinsat are not bounded suboptimal algorithms, the expensive independence check can be eliminated. Additionally, this increases the amount of parallelization that can be achieved. Alg. \ref{alg:pinsat} begins by initializing \textit{g}-values and inserting the start state with the dummy edge ($\startstate, \dac$) in \textit{OPEN} (line \ref{alg:pinsat/init1}-\ref{alg:pinsat/init2}). Once an edge is popped from \open (Alg. \ref{alg:pinsat}, line \ref{alg:pinsat/open_pop}), it is delegated for expansion to a spawned edge expansion thread. If all of the spawned edges are busy, a new thread is spawned as long as the total number of threads does not exceed the thread budget \numthreads (Alg. \ref{alg:pinsat}, line ~\ref{alg:pinsat/spawn}). When \numthreads is higher than the number of edges available for expansion at any point in time, only a subset of available threads are spawned. This prevents performance degradation due to the operating system overhead of managing unused threads.

% Let $\textbf{s}_L\in \mathcal{X}_L$ and $\textbf{s}_H\in \mathcal{X}_H$ be the low-dimensional and high-dimensional state. The algorithm takes as input the high-dimensional start and goal states $\textbf{s}_H^{start}, \textbf{s}_H^{goal}$ and recovers their low-dimensional counterparts $\textbf{s}_L^{start}, \textbf{s}_L^{goal}$ (lines \ref{line:ldrecstart}-\ref{line:ldrecend}). The low-dimensional free space $\mathcal{X}_L\setminus\mathcal{X}^{obs}$ is discretized to build a graph $\mathcal{G}_L$ to search. To search in $\mathcal{G}_L$, we use weighted A* (WA*)\cite{pohlwastar} which maintains a priority queue called OPEN that dictates the order of expansion of the states and the termination condition based on \textproc{Key($\textbf{s}_L$)} value (lines \ref{line:key}, \ref{line:term}). Alg. \ref{alg:topus} maintains two functions: cost-to-come $g(\textbf{s}_L)$ and a heuristic $h(\textbf{s}_L)$. $g(\textbf{s}_L)$ is the cost of the current path from the start state to $\textbf{s}_L$ and $h(\textbf{s}_L)$ is an underestimate of the cost of reaching the goal from $\textbf{s}_L$. WA* initializes OPEN with $\textbf{s}_L^{state}$ (line \ref{line:init}) and keeps track of the expanded states using another list called CLOSED (line \ref{line:closed}). 

The pseudocode for asynchronous edge expansion is given in Alg. \ref{alg:pinsat_expand}. If the edge selected for expansion is a dummy edge, then it is first replaced by unevaluated yet real edges in \textit{OPEN} using the low-D action set $\Aset$ (line \ref{alg:pinsat_expand:dummy1}-\ref{alg:pinsat_expand:dummy2}) and the priority value of its source state. Otherwise, the successor edge is generated and the algorithm enters the trajectory optimization routine (line \ref{alg:pinsat_expand/real1}-\ref{alg:pinsat_expand/real2}) if the edge is neither in \closed nor in \be (line \ref{alg:pinsat_expand/clorbe}). 

The trajectory optimization routine (provided in Alg. \ref{alg:pinsat_opt}) receives the set of ancestors of the edge being expanded (line \ref{alg:pinsatopt/anc}). We solve the optimization problem described in the previous sections (Eq. \ref{eq:tobj}) to find a corresponding full-dimensional trajectory from the start to the successor. This optimization is carried out in two steps. Only the trajectory generation between the ancestor node and the successor node, which is typically very short-horizon, must be optimized from scratch (line \ref{alg:pinsatopt/scratch1}, \ref{alg:pinsatopt/scratch2}). Finding the entire trajectory to the successor from the start state can be warm-started (line \ref{alg:pinsatopt/warmopt}, \ref{alg:pinsatopt/warm}) using the incoming trajectory to the ancestor. This is possible as a result of the dynamic programming in the low-D search which guarantees that an edge out of node which does not contain a valid incoming trajectory will never be expanded. In heuristic search, edge evaluation refers to computing the cost and representation of an edge that is part of the underlying graph. The optimization step of rewiring to a better parent node, common in many search algorithms, is restricted to the edges within the graph. In contrast, a key distinction in INSAT and therefore PINSAT is that it optimizes for trajectories between nodes not connected by an edge in the low-D graph when searching for a full-D trajectory to a node. INSAT generates this \textit{runtime} edge by searching over the set of ancestor nodes that lead to the node picked for expansion. Consequently, to make \weas and \epase compatible with PINSAT, we search the set of ancestor edges leading to the edge chosen for expansion. Trajectory optimization has access to the collision checker (implemented through callback mechanisms) and runs until convergence or collision occurrence, whichever comes first. During every iteration of the optimization, the iterates are reconstructed and checked for collision. In the case of collision, if the most recent iterate satisfies all the constraints, then it is returned as a solution. If the optimized trajectory is invalid, the algorithm moves to the next ancestor (line \ref{alg:pinsatopt/anc}).
\pdfoutput=1

\begin{algorithm}[H]
% \caption{\label{alg:pinsat_expand} \pinsat: Trajectory Optimization}
\caption{\pinsat: Trajectory Optimization}
\label{alg:pinsat_opt}
\begin{footnotesize}
\begin{algorithmic}[1]
\PROCEDURE{GenerateTrajectory}{$\state, \nextstate$}
    \FOR[From \startstate to \state]{$\state'' \in \textsc{Ancestors}(\state)\cup\state$} \label{alg:pinsatopt/anc}
        % \IF{$\nextnextstate = \startstate$}
        %     \STATE $\lambda^{-1}(\nextnextstate)=\startstate$
        % \ENDIF
        % \IF{$\nextnextstate = \goalstate$}
        %     \STATE $\lambda^{-1}(\nextnextstate)=\goalstate$
        % \ENDIF        
        \STATE $\traj_{\nextnextstate\nextstate} = \optimize(\nextnextstate, \nextstate)$ \label{alg:pinsatopt/opt} \COMMENT{Eq. \ref{eq:uobj}} \label{alg:pinsatopt/scratch1}
        \IF{$\traj_{\nextnextstate\nextstate}$ is collision free}
            \STATE $\traj_{\startstate\nextstate} = \warmoptimize(\traj_{\startstate\nextnextstate}, \traj_{\nextnextstate\nextstate})$ \label{alg:pinsatopt/warmopt} \COMMENT{Eq. \ref{eq:uobj}}
            \RETURN $\traj_{\startstate\nextstate}$
        \ELSIF{$\nextnextstate = \state$}
            \IF[Solve Eq. \ref{eq:tobj}]{$\traj_{\nextnextstate\nextstate}^{t_{\text{max}}}$ exists}\label{alg:pinsatopt/exists}
                \STATE $\traj_{\nextnextstate\nextstate} = \optimize(\nextnextstate, \nextstate)$ \COMMENT{Eq. \ref{eq:uobj} with $\kmin$}\label{alg:pinsatopt/kminopt} \label{alg:pinsatopt/scratch2}
                \STATE $\traj_{\startstate\nextstate} = \warmoptimize(\traj_{\startstate\nextnextstate}, \traj_{\nextnextstate\nextstate})$ \COMMENT{\ref{eq:uobj}} \label{alg:pinsatopt/warm}
                \RETURN $\traj_{\startstate\nextstate}$
            \ENDIF
        \ENDIF
        \RETURN NULL
    \ENDFOR
\ENDPROCEDURE
\end{algorithmic}
\end{footnotesize}
\end{algorithm}

\subsection{Theoretical Analysis} In this section, we will show that PINSAT is a provably complete algorithm, \textit{i.e.} it finds a solution if one exists. We will begin with some preliminaries required to prove completeness. 

\begin{definition}
    \textbf{Tunnel around Low-D Edges and Paths:} For an edge $e\in\Edges$ in $\graph$, consider a low-D subspace corresponding to that edge called low-D tunnel of the edge $\tau_L(e) \subset \lowdspace$. The full-D tunnel for that edge is given by $\tau(e) = \tau_L(e) \times \auxspace$. This definition can be extended to paths on the graph. Let a path between two nodes $\state^\prime, \state^{\prime\prime} \in \Vertices$ be given as $\gpath_{\state^\prime\state^{\prime\prime}} = \{(\state^0, \state^1), (\state^1, \state^2), \ldots, (\state^{W-1}, \state^W) \mid \state^0=\state^\prime \land \state^W=\state^{\prime\prime} \land (\state^{w-1},\state^w)\in\Edges, 1\leq w\leq W\}$. Then a tunnel around $\gpath_{\state^\prime\state^{\prime\prime}}$ is given as $\tau(\gpath_{\state^\prime\state^{\prime\prime}}) = \cup_{e\in\gpath_{\state^\prime\state^{\prime\prime}}} \tau(e)$. 

    % Let $\theta_{\textbf{x}^1\textbf{x}^2}$ denote a path in the the low-D graph $G$ such that 
% Let a path between two nodes $\state^\prime, \state^{\prime\prime} \in \Vertices$ be given as $\gpath_{\state^\prime\state^{\prime\prime}} = \{\state^0, \state^1, \ldots, \state^W \mid \state^0=\state^\prime \land \state^W=\state^{\prime\prime} \land (\state^w,\state^{w+1})\in\Edges, 0\leq w\leq W-1\}$. 
\end{definition}

\begin{assump}
\label{ass:exists}
    % \textbf{A Low-D Tunnel Contains $\traj_{\textbf{x}^S\textbf{x}^G}$:} 
    We assume that there exists at least a path on the low-D graph $G$ from $\state^S$ to $\state^G$ such that its full-D tunnel contains $\traj_{\textbf{x}^S\textbf{x}^G}$, \textit{i.e.} $\exists \gpath_{\textbf{x}^S\textbf{x}^G} \in G \mid \exists \traj_{\textbf{x}^S\textbf{x}^G} \in \tau(\gpath_{\textbf{x}^S\textbf{x}^G})$. 
\end{assump}
\begin{assump}
    $\forall e \in \Edges$, $\tau(e)$ is convex.
\end{assump}

% \begin{figure*}
% \centering
%  \includegraphics[width=1\textwidth]{img/blocking.pdf}
% \caption{PINSAT used to accelerate preprocessing of a motion library for rapid lookup. Here an ABB arm with a shield attached to its end-effector executes a motion generated by PINSAT to block a ball thrown at it.}
% \label{fig:shielding_abb}
% \vspace{-0.3cm}
% \end{figure*}

\begin{figure*}
\centering
 \includegraphics[width=1\textwidth]{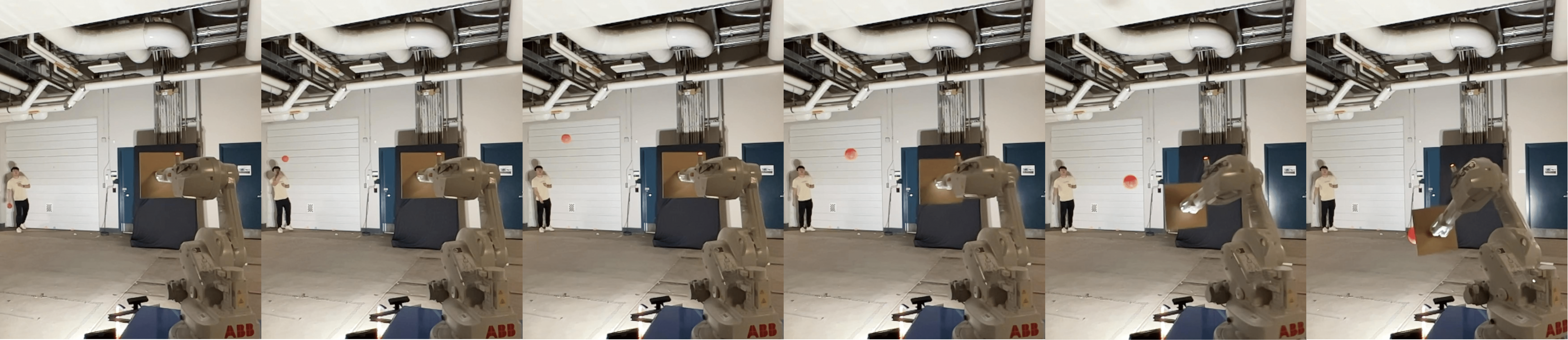}
\caption{PINSAT used to accelerate preprocessing of a motion library for rapid lookup. Here, an ABB arm with a shield attached to its end-effector executes a motion generated by PINSAT to block a ball launched toward it.}
\label{fig:shielding_abb}
\vspace{-0.3cm}
\end{figure*}

\begin{table*}
\footnotesize
\centering
\resizebox{\textwidth}{!}{%
\begin{tabular}{ccccc|cccc|c|cccc}
\toprule 
                                                    &   \multicolumn{4}{c|}{\rrtconnect+\trajopt}                          &  \multicolumn{4}{c|}{\wepase+\trajopt}                         &  \multicolumn{1}{c|}{\insat}     &    \multicolumn{4}{c}{\pinsat}                           \\\midrule
Threads                             & 5              &  10                 & 50                & 120             & 5              & 10               & 50               & 120             &  1              & 5      & 10               &  50                           & 120            \\\midrule
Success rate ($\%$)                 & 1            &  1                & 3               & 2             & 0             & 0              & 6               & 2              &  58             & 72       & 81               &  90                           & 90             \\
 Time ($\si{\second}$)      & 0.03$\pm$0.02  &  0.021$\pm$0.02     & 0.02$\pm$0.01     & 0.02$\pm$0.01   & 0.38$\pm$0.35  & 0.30$\pm$0.31    & 0.28$\pm$0.32    & 0.27$\pm$0.30   &  2.02$\pm$2.01  & 0.70$\pm$0.91       & 0.43$\pm$0.6     &  0.43$\pm$0.31                & 0.99$\pm$0.72   \\
Cost                                & 5.18$\pm$2.89  &  4.88$\pm$2.75      & 4.55$\pm$2.48     & 4.25$\pm$2.24   & 9.19$\pm$4.23  & 8.71$\pm$4.00    & 8.34$\pm$3.89    & 8.23$\pm$3.76   &  1.55$\pm$2.68  & 0.51$\pm$1.11       & 1.10$\pm$2.06    &  1.47$\pm$2.54                & 1.57$\pm$2.62  \\
\bottomrule
\end{tabular}
}
\caption{Mean and standard deviation of planning time and cost for \pinsat and the baselines for different thread budgets.}
\label{tab:pinsat_stats}
\end{table*}

\subsubsection{Minimum Order of B-spline} The minimum order $\kmin$ of the B-spline $\traj_{\kmin}(t)$ with control points $\textbf{P} = \{\textbf{p}_0, \ldots, \textbf{p}_m\}$ that is (i) guaranteed to lie entirely within the full-D tunnel of any edge $e=(\state^\prime, \state^{\prime\prime})\in\Edges$ and (ii) satisfy constraints on its derivatives (similar to Eq. \ref{eq:uobj4}) can be found by solving the simple optimization below
\begin{subequations}
\begin{alignat}{2}
\kmin =\ \ & \max_{e\in\Edges} \ \min_{k_e} \ \ k_e \label{eq:kobj1}\\ 
% \textrm{s.t.} \quad & \textbf{p}_0, \ldots, \textbf{p}_{k_e} \in \tau(e) \label{eq:kobj2}\\
\textrm{s.t.} \quad & \textbf{p}_0 = \state^\prime, \textbf{p}_m = \textbf{x}^{\prime\prime} \label{eq:kobj2}\\
& \traj_{\kmin}(t) \in \tau_L(e); 0\leq t\leq t_{\text{min}} \label{eq:kobj2.5}\\
& |\textbf{p}_0^{(j)}|, |\textbf{p}_{k_e-j}^{(j)}| = t_{\text{min}}^j\traj^{(j)}_{\text{lim}} \label{eq:kobj3}\\
& |\textbf{p}_1^{(j)}|, \ldots, |\textbf{p}_{k_e-j-1}^{(j)}| \leq t_{\text{min}}^j\traj^{(j)}_{\text{lim}} \label{eq:kobj3}
\end{alignat}
\label{eq:kobj}
\end{subequations}    
The tunnel constraints in Eq $\ref{eq:kobj2.5}$ are transcribed into constraints on the control points similar to how corridor constraints are represented in \cite{bspcorridor}. In this work, the outer maximization is carried over actions set $\Aset$, which is typically a significantly smaller set than $\Edges$. This can be done if the edges in the graph can be grouped as a finite number of action primitives sharing the same tunnel representation. Note that the inner optimization in the above formulation is a linear program and is guaranteed to have at least one solution as long as the feasible region of the decision variable is nonempty and bounded. 
\begin{remark}
The feasible region depends on the limits of the system $\traj^{(j)}_{\text{lim}}$ and the choice of the tunnel $\tau$ which is domain-specific. However, in most cases, a straightforward design of the tunnel \cite{bspcorridor} can ensure a nonempty feasible region and a reasonably low $\kmin$.     
\end{remark}
Consider an edge $e = (\state^\prime, \state^{\prime\prime})\in\Edges$. 
\begin{lemma}
    If $\tau(e)$ is convex, then a $\traj_{\state^\prime\state^{\prime\prime}}^\sim(t) \in \tau(e), \forall t\in[0,t_{\text{max}}]$ that solves Eq. \ref{eq:tobj} to global optimality can be found. 
\end{lemma}
\begin{proof}
    Setting $\freespace=\tau(e)$ make Eq. \ref{eq:tobj5} a convex domain constraint on the decision variables and Eq. \ref{eq:tobj} a convex program. Assuming $t_{\text{max}}$ is large enough, we know that a convex optimization can be solved to global optimality. 
\end{proof}
Let $\boldsymbol{\hat{\phi}}_{\state^\prime\state^{\prime\prime}}\in\mathcal{C}^\gamma(\mathbb{R})$. $\boldsymbol{\hat{\phi}}_{\state^\prime\state^{\prime\prime}}$ need not satisfy $\traj_{\text{lim}}^{(j)}, j\leq\gamma$.\begin{lemma}
    If $\boldsymbol{\hat{\phi}}_{\state^\prime\state^{\prime\prime}} \in \tau_L(e)$, then there exists a $\kmin$-th order $\traj_{\state^\prime\state^{\prime\prime}}^* \in \tau(e)$ that satisfies Eq. \ref{eq:uobj}. 
    \label{lem:bb}
\end{lemma}
\begin{proof}
    It directly follows from the constraint satisfaction for finding $\kmin$ in Eq. \ref{eq:kobj}. This $\kmin$ can be used to choose the B-spline basis for optimizing Eq. \ref{eq:uobj} simultaneously over path length and time to find  $\traj_{\state^\prime\state^{\prime\prime}}^* \in \tau(e)$.
\end{proof}
In other words, Lemma. \ref{lem:bb} states that if there is a trajectory that is (i) $\mathcal{C}^j$-continuous, (ii) satisfies the low-D boundary conditions of an edge $e\in\Edges$ and, (iii) is contained entirely within the low-D tunnel of the edge $\tau_L(e)$, then we can use $k_{\text{min}}$-th order basis functions to construct a trajectory that satisfies full-D $j$-th derivative boundary conditions and lies entirely within the full-D tunnel of the edge $\tau(e)$. We note that Lemma. \ref{lem:bb} is a critical building block towards developing PINSAT's guarantee on completeness.

% The set $\tau(e), e\in\Edges$ denoting the tunnel around any low-D edge in Eq. \ref{eq:kobj2} can be easily designed to be nonempty and bounded. We also know that $t_{\text{min}}, \traj^{(j)}_{\text{lim}} > 0$. 
% \begin{lemma}
% \begin{proof}
% \end{proof}
% \end{lemma}
\begin{theorem}
    % \textbf{Completeness of PINSAT:}  we show that under the choice of full-D trajectory representation (B-splines in our case) if the algorithm used to search the low-D is complete, then PINSAT is guaranteed to find a full-D trajectory solution if one exists.  
    \textbf{Completeness of PINSAT:} If Assumption \ref{ass:exists} holds, then PINSAT with $\kmin$-th order B-spline optimization is guaranteed to find a $\traj_{\state^S\state^G}$ that satisfies Eq. \ref{eq:obj}.  
\end{theorem}

\FloatBarrier
\section{Evaluation}

We evaluate PINSAT in a kinodynamic manipulation planning for an ABB IRB 1600 industrial robot shown in Fig \ref{fig:pinsat_abb}. The red horizontal and vertical bars are obstacles and divide the region around the robot into eight quadrants, four below the horizontal bars and four above. We randomly sample 500 hard start and goal configurations. We ensure that the end effector corresponding to one of those configurations is above the horizontal bars and the other is underneath. We also ensure that the end effector for the start and goal configurations are not sampled in the region enclosed by the same two adjacent vertical bars. Our samples are deliberately far apart in the manipulator's C-space and require long spatial and temporal plans to connect them. The motivation for sampling start-goal pairs in this manner is to subject the planner to stress tests under severe constraints on the duration of the trajectory. All the experiments are run on an AMD Ryzen Threadripper Pro with 128 threads. 

For our experiments, we use a backward breadth-first search (BFS) from the goal state in the task space of the manipulator as our heuristic. The action space $\Aset$ is made of small unit joint movements. For every joint, we had a $4\deg$ and $7\deg$ primitive in each direction subject to the joint limits. In addition to this, during every state expansion, we also check if the goal state is the line of sight to consider the goal as a potential successor. The velocity limits of the arm were inflated 10x the values specified in the datasheet for simulation. We used $50m/s^2$ and $200m/s^3$ as the acceleration and jerk limits. The maximum duration of trajectory was limited to 0.6s. The reason for inflating the velocity limits is to allow the bare minimum time to navigate the obstacles and reach from start to goal. 

Table \ref{tab:pinsat_stats} shows success rate and planning time statistics for INSAT and PINSAT for different thread budgets. The planning time statistics were only
computed over the problems successfully solved by both algorithms for a given thread budget. We also implemented a parallelized version of RRTConnect (pBiRRT) as a baseline. The kinematic plan generated by pBiRRT was then post-processed by the same B-spline optimization used in PINSAT to compute the final kinodynamic plan. As the optimization has a constraint on the maximum execution duration of the trajectory, adding all the waypoints from the path at once will result in an over-constrained system of equations for which a solution may not exist. To circumvent this, we iteratively add the waypoint constraint starting from the minimum containing just the start and the goal state. Though pBiRRT had a 100\% success rate in computing the spatial trajectory, the results after post-processing with the optimizer were abysmal. PINSAT achieves a significantly higher success rate than INSAT for all thread budgets greater than 1. Specifically, it achieves a 5x improvement in mean planning time, a 7x improvement in median planning time, and a 1.8x improvement in success rate for $N_t$ = 50. Even with a single thread, PINSAT achieves lower planning times than INSAT. This is because of the decoupling of edge evaluations from state expansions in ePA*SE.

% \subsubsection{Post Processing Waypoints with Spline Optimization}
% In this section, we describe how we post-process the path from kinematic planners like RRT to obtain a time-parameterized B-spline trajectory. Specifically, we explain how any waypoint from the path can be added as a linear constraint to the above optimization setup. 

\FloatBarrier
\section{Conclusion}

% We presented PINSAT, an extension of INSAT that asynchronously parallelizes the numerous \textit{expensive} calls to the trajectory optimizer. PINSAT does this by borrowing the ideas of edge expansion and asynchronous edge evaluation from the recently introduced parallel search algorithm ePA*SE \cite{mukherjee2022epase}. We evaluated our algorithm using a kinodynamic manipulation planning domain and demonstrated significantly higher success rates than INSAT \cite{insat, insatptc} and a drastic reduction in planning time. PINSAT only guarantees completeness and in the future, we want to explore how to establish independence checks as in ePA*SE to guarantee optimality. 

We presented PINSAT, an extension of INSAT that asynchronously parallelizes the numerous \textit{expensive} calls to the trajectory optimizer. PINSAT achieves this by incorporating the concepts of edge expansion and asynchronous edge evaluation borrowed from the recently introduced parallel search algorithm ePA*SE \cite{mukherjee2022epase}. We evaluated our algorithm using a kinodynamic manipulation planning domain and demonstrated significantly higher success rates than INSAT \cite{insat, insatptc}, along with a drastic reduction in planning time. PINSAT guarantees only completeness at present, but in the future, we aim to explore how to establish independence checks, as done in ePA*SE, to guarantee optimality."

\FloatBarrier
% \bibliographystyle{IEEEtranBST/IEEEtran}
% \bibliography{IEEEtranBST/IEEEabrv,root}
\bibliographystyle{ieeetr}
\bibliography{root}


% \section*{ACKNOWLEDGMENT}

% The preferred spelling of the word ÒacknowledgmentÓ in America is without an ÒeÓ after the ÒgÓ. Avoid the stilted expression, ÒOne of us (R. B. G.) thanks . . .Ó  Instead, try ÒR. B. G. thanksÓ. Put sponsor acknowledgments in the unnumbered footnote on the first page.

%%%%%%%%%%%%%%%%%%%%%%%%%%%%%%%%%%%%%%%%%%%%%%%%%%%%%%%%%%%%%%%%%%%%%%%%%%%%%%%%

% References are important to the reader; therefore, each citation must be complete and correct. If at all possible, references should be commonly available publications.

% \begin{thebibliography}{99}
% \end{thebibliography}

\end{document}